\documentclass{article}
\usepackage[nonatbib,preprint]{neurips_2019} 

\usepackage[utf8]{inputenc} 
\usepackage[T1]{fontenc}    
\usepackage{hyperref}       
\usepackage{url}            
\usepackage{booktabs}       
\usepackage{amsfonts}       
\usepackage{nicefrac}       
\usepackage{microtype}      

\RequirePackage{amsmath,amssymb,amsfonts,amsthm,bbm}
\usepackage{enumerate}
\usepackage{amssymb}
\usepackage{mathrsfs}
\usepackage{amsthm}
\usepackage[utf8]{inputenc}
\RequirePackage[OT1]{fontenc}
\usepackage{enumerate}
\usepackage{amssymb}
\usepackage{amsfonts} 
\usepackage{amsmath}
\usepackage{accents}
\usepackage{color}
\usepackage{graphicx}
\usepackage{dsfont}
\usepackage{multirow}
\usepackage{mathtools}
\usepackage[usenames,dvipsnames,svgnames,table]{xcolor}
\usepackage{color,graphicx}
\usepackage{algorithmicx,algpseudocode, algorithm}
\usepackage{hyperref}
\usepackage{appendix}
\definecolor{darkgreen}{rgb}{0,0.6,0}
\definecolor{red}{rgb}{0.7,0.15,0.15}

\definecolor{darkblue}{rgb}{0,0,0.6}
\hypersetup{
	colorlinks=true,
	linkcolor=darkgreen,
	citecolor= darkblue,
	urlcolor = black}
\DeclareMathOperator*{\argminA}{argmin}

\newtheorem{theorem}{Theorem}[section]
\newtheorem{lemma}[theorem]{Lemma}

\newtheorem{property}{Property}

\newtheorem{assumption}{Assumption}
\theoremstyle{definition}

\theoremstyle{remark}

\newcommand{\by}{\mathbf{y}}

\newcommand{\R}{\mathbb{R}^{d}}

\newcommand{\w}{\textbf{w}}

\numberwithin{equation}{section} 

\title{Online Graph-Based Change-Point Detection for High Dimensional Data}

\author{%
Yang-Wen Sun 
  \\
 Humboldt University Berlin\\
 \texttt{yangwen.sun@hu-berlin.de} \\
  \And
Katerina Papagiannouli\\
Humboldt University Berlin\\
\texttt{papagiai@hu-berlin.de} \\
\And
Vladmir Spokoiny \\
Humboldt University Berlin\\
Weierstrass Institute Berlin\\
\texttt{spokoiny@wias-berlin.de} \\
}

\begin{document}

\maketitle

\begin{abstract}
Online change-point detection (OCPD) is important for application in various areas such as finance, biology, and the Internet of Things (IoT). However, OCPD faces major challenges due to high-dimensionality, and it is still rarely studied in literature. In this paper, we propose a novel, online, graph-based, change-point detection algorithm to detect change of distribution in low- to high-dimensional data. We introduce a similarity measure, which is derived from the graph-spanning ratio, to test statistically if a change occurs. Through numerical study using artificial online datasets, our data-driven approach demonstrates high detection power for high-dimensional data, while the false alarm rate (type I error) is controlled at a nominal significant level. In particular, our graph-spanning approach has desirable power with small and multiple scanning window, which allows timely detection of change-point in the online setting. 
\end{abstract}

\section{Introduction}

Change-point detection has been widely applied in various fields such as finance \cite{spokoiny2009multiscale}, biology \cite{chen2011parametric}, and the Internet of Things (IoT) \cite{aminikhanghahi2018real}. Nowadays, as sensing and communication technologies evolves, high-dimensional data are generated seamlessly in various fields. Hence, high dimensionality, online (timely), and algorithm robustness constitute major challenges to modern change-point detection problem. \par

Statistically, a change-point can be characterized as a point in sequential data $y_i, i=1, 2, \dots$, where the probability distribution prior- and after- the sequential data are different, that is $\exists \tau > 0,  y_i \sim F_0$, for $ i< \tau$, otherwise $y_i \sim F_1$. Traditional parametric approaches have limitation for the high-dimensional data as the number of parameters to be estimated surpass the number of observations available and also assumptions needed for the distribution of each individual dimension are difficult as the underlying distributions are normally highly context specific \cite{siegmund2011detecting}. On the other hand, for the nonparametric approaches such as kernel-based method \cite{harchaoui2009kernel}, the increasing dimension makes the selection of kernel function and the bandwidth a complicated optimization process. 

To reduce the complexity of CPD problem due to dimensionality, the graph-based CPD is devised. A similarity measure is introduced in the graph-based CPD.  This similarity measure transforms the dimensional data into dimensionless metrics, which, firstly, alleviates the curse of dimensionality and secondly, is a test statistic to compare the distribution statistically. Further two-sample test is carried for hypothesis testing. We bring the graph-based CPD online to perform detection in real-time setting, while offline detection retrospectively detect changes in a closed dataset. Multiple examination windows incept the incoming data for a timely detection. Our proposed graphed-based change-point detection framework allows us to detect changes online while maintains accuracy with small scanning window. In this paper, we entails the graph-based CPD algorithms, provide theoretical base for the algorithm, and then give empirical validation of these results.

\subsection{Related work}\label{subsec:Related}
Graph-based CPD methods are two-sample tests based on various types of graphs representing the similarity between observations,  which are first proposed by Friedman and Rafsky \cite{friedman1979multivariate} (1979) using minimum-spanning tree (MST) graph. Also, Rosenbaum \cite{rosenbaum2005exact} (2005) propose another test based on the minimum-distance pairing (MDP) using the rank of the distance within the pairs, which is thus restricted to MDP graph. Recently, Chen and Zhang \cite{chen2015graph} (2015), utilize MST and MDP graph representations onto the data, and construct a test statistic based on counting the number of edges connecting points from different groups. They demonstrate that graph-based CPD has better detection power at high-dimensional data compared to parametric methods, such as Hotelling’s $T^2$ and generalized likelihood ratio test. Chen and Zheng develop an algorithm to count the edges between groups before and after the potential change-point. However, this pioneer CPD graph method, firstly, works only for offline detection, which is not sufficient in fields where an immediate reaction is needed when a change-point is detected. Secondly, its detection power is comparably not sensible to the variance change and is limited by the size of the dataset. Our proposed method contributes in filling the aforementioned gaps.  

\section{Online change-point detection based on graph similarity} 
We observe $\{\by_i\}, i=1, 2, \dots$, where $y_i \in \mathds{R}^d$ with number of observation not fixed as we are receiving data online. The change-point problem can be formulated as hypothesis testing problem, that is, to test the null hypothesis
\begin{equation}\label{H0}
H_{0}:  \quad \by_i \sim F_0, \quad i=1, 2, \dots. 
\end{equation}
against the single change-point alternative
\begin{equation}  \label{H1}
	H_{1}: \quad \exists ~ 1 < \tau,\ \by_i \sim \left\{ \begin{array}{ll} F_1, & i> \tau \\ F_0, & \text{otherwise}, \end{array} \right.
\end{equation}
where $F_0$ and $F_1$ are two probability measures that differ on a set of non-zero measure. $\tau$ refers to the change-point and data end point is not written here since we observe the data sequentially.

\subsection{Graph similarity }\label{subsec:graph_simi}
We consider an undirected graph $G=(V,E)$, in which vertices $V=[n]$ represent a block of $n$ consecutive observations from the sequential data, and edges $E$ indicates the connectivity of two end nodes, and edge weights $W_{i,j}$ as the squared Euclidean distance between vertices. Possible choices of graphs are complete graph, minimum-spanning tree (MST), minimum-distance pairing (MDP), and nearest-neighbor graph (NNG) \cite{chen2017new}. 

\subsection{Test statistic for change-point detection}
As dimension $d$ of observation increases, it becomes challenging to compare distribution $F_0$ to $F_1$ as mentioned in (\ref{H1}). Instead, to test the null $H_0$ (\ref{H0}) versus the single change-point alternative $H_1$ (\ref{H1}), we devise the test based on a similarity measure, i.e. test statistic, derived from the graph structure from the sample space of observations $\{\by_i\}$.  For any candidate value $ t $ of the change-point $ \tau $, we set up an online scanning window  $\{\by_i:~i=1,\dots, n\}$, where $ n \in \mathbb{N}$, and derive the test statistic based on the graph similarity of the data. In each scanning window, we divide the observations into two equally large groups: observations come before $t$ and observations that come after $t$, i.e. the potential change-point. We construct three graphs, as described in Section \ref{subsec:graph_simi}, based on segments from the scanning window. Let $n \in \mathbb{N}, n$ is even and $G$ be the graph on  $\{\by_i:~i=1,\dots, n \}$, $G_1 $ the graph be the graph on  $\{\by_i:~i=1,\dots, n/2\}$, and $ G_2 $ be the graph on  $\{\by_i:~i=n/2+1,\dots, n\}$. Then, we introduce the following Graph Spanning Ratio (GSR) test statistic for the graphical mean:
\begin{equation} \label{eq:T_mu}
T_{\mu,n}= \frac{dG}{dG_1+dG_2}
\end{equation}

and test statistic for the graphical variance:
\begin{equation}\label{eq:T_var}
  T_{\sigma,n}= \frac{dG_1}{dG_2}+\frac{dG_2}{dG_1} 
\end{equation}  
We denote the sum of edge weights for graph $G$ by $dG$, that is $ d{G}= \sum_{\{ij\} \in E_{G}} W_{ij} $, where  $W_{ij}  = \|\by_{i} -\by_{j}\|^{2}$, similarly for $dG_1, dG_2$. The similarity graphs $ G, G_1, G_2 $ are constructed with the square Euclidean distance between the nodes using selected graph type. Note that spanning ratio of the graphical mean, $T_{\mu,n}$ is devised in such way that it increases when there is a change in mean of the distribution $F_0, F_1$, as an example shown in Figure \ref{fig:MSTs}. Similarly, when there is a change in variance, then we can see that the spanning ratio of graphical variance,$T_{\sigma,n}$, increases as the spanning distances of graph $G_1$ and $G_2$ varies.

\begin{figure}
  \centering
  \textbf{a}
  \includegraphics[height=.21\textheight]{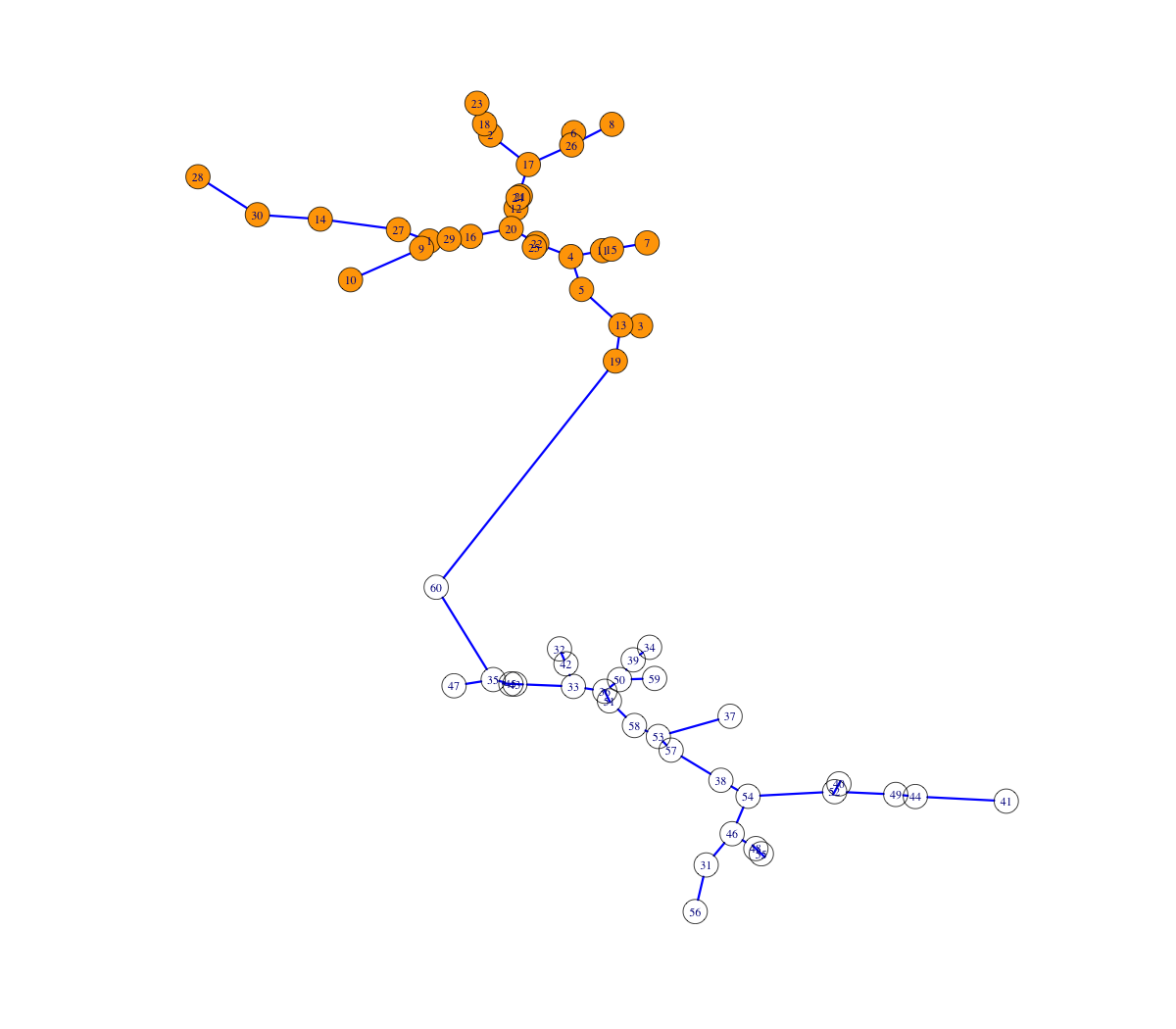}
  \textbf{b}
  \includegraphics[height=.21\textheight]{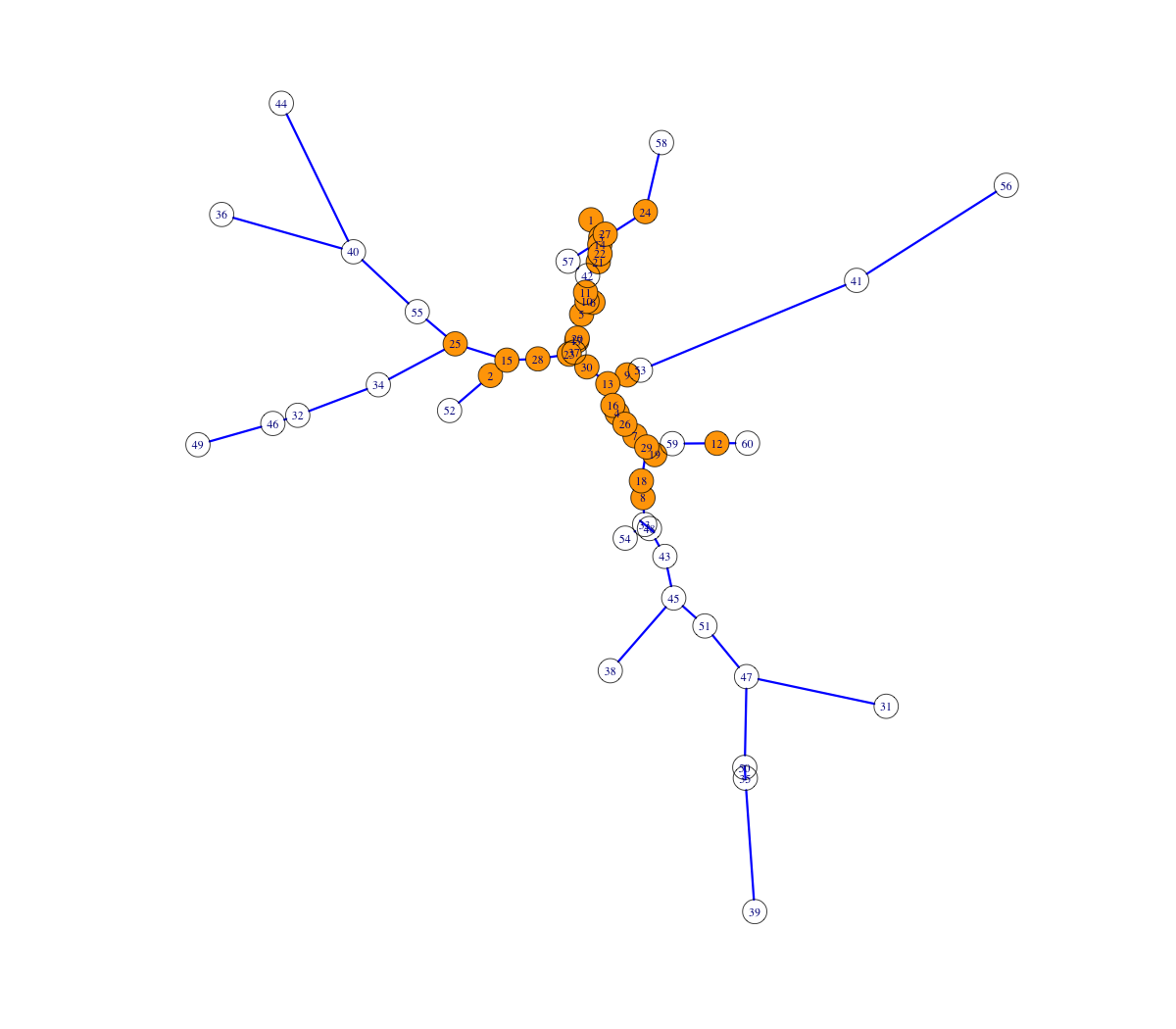}\setlength{\belowcaptionskip}{0pt}
  \caption{Graph representation of a two-dimensional sequential data. Minimum spanning tree graphs are constructed from 60 i.i.d. normal distributed observations with first 30 observations (in orange color) from standard normal, the second 30 observations (in white color) with (a) change in mean, (b) change in variance. }
 \label{fig:MSTs}
\end{figure}

Two-sample hypothesis tests with test statistics $T_{\mu,n}$ and $T_{\sigma,n}$ are used to detect change in graphical mean and variance. Given a significant level $\alpha$, we reject $H_0$  if
                  \[ T_{\mu,n} > \rho_{\alpha,\mu,n} \]
                  where the critical point to test the change of graphical mean
                  \[ \rho_{\alpha,\mu,n}= \argminA_{\rho} \{\mathds{P}\big( T_{\mu,n}  \geq \rho \big) \leq \alpha \} \]
or 
                  \[ T_{\sigma,n} > \rho_{\alpha,\sigma,n}\]
                  where the critical point to test the change of graphical variance
                  \[ \rho_{\alpha,\sigma,n} = \argminA_{\rho} \{\mathds{P}\big( T_{\sigma,n}  \geq \rho \big) \leq \alpha \} \]

The critical values $ \rho_{\alpha,\mu,n}, \rho_{\alpha,\sigma,n}$ under the null distribution $H_0$ can be approximate by the limiting behavior of the test statistics. In practice, the limit distribution often neither can be  expressed analytically, nor the convergent rate to limit distribution is fast. For sequential data, the limit distribution could in the case that small sample dependency structure would not sufficient to take into account \cite{kirch2008bootstrapping}. As a result, we study the critical value $\rho_{\alpha,\mu,n}, \rho_{\alpha,\sigma,n}$ using permutation method. 

Under the assumption that observations are independent and identically distributed from null distribution $H_0$, the joint distribution of a block of $n$ observations, $\{\by_i:~i=1,\dots, n\}$ is the same when we permute the order of the sequential data. In the permutation method, we define the null distribution of GSR test statistics to be the permutation distribution which encompass $ n! $ possible permutations of  $\{\by_i:~i=1,\dots,n\}$. For each permutation we perform a GSR test statistics. Through the permutation procedures, we recover the empirical distribution for the GSR test statistics. What is the threshold of the GSR test statistics (\ref{eq:T_mu}) need to be, to draw sufficient evidence against the null hypothesis (\ref{H0}). In that sense, we want to find the critical value from the empirical distribution of the test statistics under (\ref{H0}), that is

\begin{equation}\label{eq:rhoM}
\rho^b_{\alpha,\mu,n} = \argminA_\rho \{\mathds{P}^b\big( {T_{\mu,n}^b}  \geq \rho \big) \leq \alpha \} 
\end{equation}
\begin{equation}\label{eq:rhoV}
\rho^b_{\alpha,\sigma,n} = \argminA_\rho \{\mathds{P}^b\big( {T_{\sigma,n}^b}  \geq \rho \big) \leq \alpha \} 
\end{equation}

where  $\rho^b_{\alpha,\mu,n}$ and  $\rho^b_{\alpha,\sigma,n}$ denote the critical values for test statistics $T_{\mu,n}$ and $T_{\sigma,n}$ and $\mathds{P}^b$ is the probability measure under the empirical distribution. Next, we propose an algorithm for static GSR change-point detection, StaticGSR($Y_{0}, Y_{}, b, \alpha$).  We assume observations in $Y_{0}$ are i.i.d. and are no change-point. Based on the permutation procedure, it computes the critical value of the GSR test statistic (\ref{eq:T_mu}) for each window length from the training data $Y_{0}$. The empirical critical value calculated from this training observations will then apply to further change-point detection. One thing to mention is that this method detects whether a change-point exist in the middle of the new observations.

\begin{algorithm}
	\caption{Static graph-spanning-ratio change-point detector: StaticGSR($Y_{0}, Y_{}, k, \alpha$)}\label{algo.offline}
	\begin{algorithmic}
				\State \textbf{Input:} data $Y_{0}$ of size $n$, permutations $ k $\\
				\Comment For each window length $n$, we estimate the empirical distribution of test statistics by permutation
				\For {$ i = 1, \dots, k $} 
					\State Permute $Y_{1}$	
					\State Compute $ {T_{\mu,n}}^{i} = \frac{dG}{dG_1+dG_2}, {T_{\sigma,n}}^{i} = \frac{dG_1}{dG_2}+\frac{dG_2}{dG_1} $
			     \EndFor
			     \State Compute $ \rho^b_{\alpha,\mu,n}, \rho^b_{\alpha,\sigma,n} $ \Comment critical value
	\end{algorithmic}\par
	
     \begin{algorithmic}
      \State \textbf{Input:} data $Y_{}$ of size $n$       
           \State Compute  $ {T_{\mu,n}} = \frac{dG}{dG_1+dG_2}, {T_{\sigma,n}} = \frac{dG_1}{dG_2}+\frac{dG_2}{dG_1} $
			\If {${T_{\mu,n}}  \geq \rho^b_{\alpha,\mu,n}  ||  {T_{\sigma,n}}  \geq \rho^b_{\alpha,\sigma,n} $} 
				\State Reject $ H_{0} $ \Comment CP detected
			\Else 	
				\State Accept $ H_{0} $ \Comment No CP detected
			\EndIf
	  \end{algorithmic}
\end{algorithm}

\subsection{Online graph change-point detection}
In the online setting, we receive observation consecutively. Hence, we need to consider online calibration of critical value for change-point detection. Under the i.i.d and no change-point assumption, we take in the first incoming $N$ observation $y_1,...,y_N$ as our training sample. Here we define a zone $A_n=\{n/2-1, ..., N-n/2\}$, where $n$ is the size of the scanning window. For online change-point detection, we need to consider multiple windows detection which later would allows us to capture the change-point in a timely manner. Assume observations are i.i.d. Similarly to the static CPD procedure mentioned above, we apply permutation on the training sample to calculate the empirical critical value for GSR test. In this procedure, we permute the sample $b$ times, and based on which we construct the permutation test statistics. Here, we write the test statistics as, for $t \in A_n $,
		\[  {T_{\mu,n}^b}(t) = \frac{{dG}^b}{{dG_1}^b+{dG_2}^b} \]
		\[  {T_{\sigma,n}^b}(t)  = \frac{{dG_1}^b}{{dG_2}^b}+ \frac{{dG_2}^b}{{dG_1}^b} \]
	where 
	     ${dG}^b, {dG_1}^b, {dG_2}^b$ are the spanning distances from Graph $G, G_1, G_2$. For $t \in A_n $, Graph $G$ contains data points from $y_{t-n/2+1}, \dots,y_{t+n/2}$, Graph $G_1$ contains data points from $y_{t-n/2+1}, \dots, y_t$, and Graph $G_1$ contains data points from $y_{t+1}, \dots,y_{t+n/2}$. 
 To find the empirical critical value, for each permutation, we calculate
		\[ T_{\mu,n}^b:= \max_{t \in A_n} T_{\mu,n}^b(t) \]
	     \[ T_{\sigma,n}^b:= \max_{t \in A_n} T_{\sigma,n}^b(t)\]
		
For each window length $n$, we compute the quantile function 
         \[ {Z}_{\mu,n}^b(z) :=inf\{ x: \mathds{P}^b	\big( T_{\mu,n}^b  \geq x \big) \leq z \}\]
         \[ {Z}_{\sigma,n}^b(z) :=inf\{ x: \mathds{P}^b	\big( T_{\sigma,n}^b  \geq x \big) \leq z \}\]
For each window length $n$, the critical value based on the permutation method is specified as in (\ref{eq:rhoM}) and (\ref{eq:rhoV}), which gives
    \[  \rho^b_{\alpha,\mu,n}={Z}_{\mu,n}^b(\alpha) \]
    \[  \rho^b_{\alpha,\sigma,n}={Z}_{\sigma,n}^b(\alpha) \]

Permutation calibration is used for the online setting to control the false alarm rate \cite{avanesov2018change}. To lower the false alarm rate, we calibrate for all window length $\mathfrak{N}$
    \[  \alpha^* := sup\{ z : \exists n \in \mathfrak{N}, i \in\{\mu, \sigma \}; \mathds{P}^b\big(T_{i,n}^b > {Z}_{i,n}^b(z) \big) <  \alpha\} \]    

\begin{algorithm}
	\caption{Online graph change-point detector: OnlineGSR($Y, N, n, k,\alpha$)}\label{algo.online}
			\begin{algorithmic}
				\State \textbf{Input:} initial $N$ observations from data $Y$, window length $n$, permutations $ k$
				\For {$ i = 1, \dots, k $}
					\State Permute $Y_1$	
					\For {$ t = n/2, \dots, N-n/2$}
					\State Compute $ T_{\mu,n}(t) = \frac{dG}{dG_1+dG_2}, T_{\sigma,n}(t) = \frac{dG_1}{dG_2}+\frac{dG_2}{dG_1} $
					\EndFor \\
					Compute $ T_{\mu,n}^{i} = \max\limits_{t} T_{\mu,n}(t),  T_{\sigma,n}^{i} = \max\limits_{t} T_{\sigma,n}(t)$     
			     \EndFor
			     \State Compute $ \rho_{\alpha,\mu,n}, \rho_{\alpha,\sigma,n} $ \Comment critical value for window length $n$
	       \end{algorithmic} \par       
           \begin{algorithmic}
           \State \textbf{Input:} data $Y$ after $N$ observations
           \State For each window length $n$, compute  $ T_{\mu,n} = \frac{dG}{dG_1+dG_2}, T_{\sigma,n} = \frac{dG_1}{dG_2}+\frac{dG_2}{dG_1} $ 
			\If {For any window length $n$, $T_{\mu,n} \geq \rho_{\alpha,\mu,n}  ||  T_{\sigma,n} \geq \rho_{\alpha,\sigma,n} $} 
				\State Reject $ H_{0} $ \Comment CP detected							
			\Else 
				\State Accept $ H_{0} $ \Comment No CP detected
			\EndIf
		\State \textbf{Output:} change-point not detected / detected at $\frac{n}{2}$ before last observation.
			\end{algorithmic}
\end{algorithm}

\section{Theoretical properties of the GSR test statistics }
	We assume $ N$ is the number of observations and $ n $ is the scanning window length, where $ n < N $ and $ n $ is an even number.
	\begin{assumption}\label{gr}
		$ G$ is a complete graph with $ n $ nodes. By this we mean that the graph $ G $ is undirected and every pair of distinct vertices is connected by a unique edge.
	\end{assumption}
	
	\begin{assumption}\label{nodes}
		The nodes of  graph  $ G $, namely $ \by_i$, are i.i.d. random variables normally distributed, $ \by_i \sim\mathcal{N}\left(\mu, \sigma^{2} I \right)$ and $ \by_i \in \R$.
	\end{assumption}

\begin{assumption}\label{dist}
	The nodes are paired with the euclidean distance in $ \R $. 
\end{assumption}
	Next, we show some theoretical properties concerning the distribution of the GSR test statistics  and the spanning distance of the graph.
	\begin{lemma} \label{spdidi}
		Let Assumptions \ref{gr}-\ref{dist} hold. Then, the distance spanned by $ G $ 
		\[ 
		dG = \sum_{i<j} \|\by_i, -\by_j\|^{2} = \sum_{i= 1}^{m} \lambda_{i} U_{i}^{2}
		 \]
		is  a linear combination of independent chi-square distributions, where $\|\cdot\|$ is the euclidean distance in $ \R $, $U $ follows a multivariate standard normal distribution. 
	\end{lemma}

\begin{property}
The distance spanned by $ G $ follows chi-square distribution with degrees of freedom $ (m-1)d $.
\end{property}
Further, we give the distribution ot the GSR mean test statistic $T_{\mu,n}$ (\ref{eq:T_mu}).
\begin{lemma}\label{testmu}
Let Assumptions \ref{gr} - \ref{dist} hold. Then, the GSR test statistics for the graphical mean follows Fisher distribution. Precisely, 
\[
T_{\mu, n} = \frac{dG}{dG_{1} + dG_{2}}\sim F(d_{1}, d_{2}),
\]
where $ dG, dG_{1}, dG_{2}$ are the distances spanned by a fully-connected graph without change-point, before and after a potential change-point respectively. $ d_{1}, d_{2} $ are the degrees of freedom.
	\end{lemma}

\begin{lemma}
	Let Assumptions \ref{gr} - \ref{dist} hold. Then, the GSR test statistic for the graphical variance follows Fisher distribution. Precisely, 
	\[
	T_{\sigma, n} = \frac{dG_1}{dG_{2}} + \frac{dG_2}{dG_1} = Z +\frac{1}{Z} 
	\]
	where $dG_{1}, dG_{2}$ are the distances spanned by a fully-connected graph before and after a potential change point respectively. Thus, $ Z \sim F(d_{3}, d_{4})$ and $ \frac{1}{Z}\sim F(d_{4}, d_{3})$. $d_{3}, d_{4}$ are the degrees of freedom. 
\end{lemma}		
	
\section{Numerical studies}
To examine the power of our proposed method, we report our results of CPD from low- to high- dimensional data, with different choices of graph, and by various lengths of scanning window. The numerical studies consist two parts: first, the comparison of offline (static) graph-based change-point detection (GBCPD) methods, and second, online detection power of our proposed method. For static GBCPD (algorithm \ref{algo.offline}). We are curious about the comparison with the other existing graph based CPD method. Thus, we select graph representation of minimum-spanning tree (MST) and complete graph (CG) in comparison to the in-between-group edge counting (IBGEC) algorithm proposed by Chen and Zheng in Section \ref{subsec:Related} with various scanning window lengths. For the online GBCPD, we study the detection power of algorithm \ref{algo.online} in relation to the data dimension and to the scanning window length. At the final step, we apply multiple-scanning window to the online detection algorithm and study the detection performance with CG and MST graphs.

\subsection{Detection power comparison of offline (static) GBCPD methods }
\begin{table}[b]
\centering{
  \caption{Confusion matrix for detection accuracy and sensitivity}
  \label{table:confusion_table}
\begin{tabular}{ccc}
\cline{2-3}
\multicolumn{1}{l|}{}                                & \multicolumn{1}{l|}{True change-point} & \multicolumn{1}{l|}{True non-change-point} \\ \hline
\multicolumn{1}{|l|}{Identified as change-point}     & \multicolumn{1}{c|}{TP}                & \multicolumn{1}{c|}{FP}                    \\ \hline
\multicolumn{1}{|l|}{Identified as non-change-point} & \multicolumn{1}{c|}{FN}                & \multicolumn{1}{c|}{TN}                    \\ \hline
\end{tabular}
}
\end{table}

To quantify the detection power of our proposed method, we consider the scenarios that the observation follow certain parametric distribution. We generate 100 samples for detection power comparison. Each sample is consist of $n$ simulated i.i.d observations, n is even. With equal probability, it follows $d$ dimensional standard normal distribution $y_i \overset{i.i.d}{\sim} \mathbf{N}(0,I_d), i =1,\dots, n$ or from the model:
\begin{equation} 
	 y_i \overset{i.i.d}{\sim} \left\{  \begin{array}{ll} \mathbf{N}(0,I_d), \quad i =1,\dots, \frac{n}{2}; \\ 
	                                                     \mathbf{N}(\Delta, \Sigma), \quad i =\frac{n}{2}+1,\dots, n. \end{array} \right.
\end{equation}\label{eq:simu_offline}

We consider both accuracy and sensitivity as a general way of comparing detection power \cite{aminikhanghahi2017survey}. Detection accuracy is defined as how often the detection algorithm make right decision, that is, to identify change-point when there in reality a true change-point, and identify no change-point when there is true non-change-point. That is accuracy$=\frac{TP+TN}{TP+TN+FP+FN}$,using notation defined in Table \ref{table:confusion_table}. We denote FPR $= \frac{FP}{FP+TN}$ as the false positive rate which is rate of giving a false alarm when no change-point present.

For detection sensitivity, we concern about the success rate of identify a change-point when there indeed true change-point exist. Therefore, sensitivity $= \frac{TP}{TP+FN}$. To consider the detection power with both the accuracy and the sensitivity of the detection methods, here we define a power metric as the geometric mean of the accuracy and sensitivity \cite{aminikhanghahi2018real}, P\_mean $= \sqrt{accuracy \times sensitivity}$

\begin{figure}[t]
 \hspace {6pt}   sGSR$_{CG}$ \hspace{80pt}       sGSR$_{MST}$  \hspace{60pt}   IBGEC
  \centering 
  \includegraphics[height=.15\textheight]{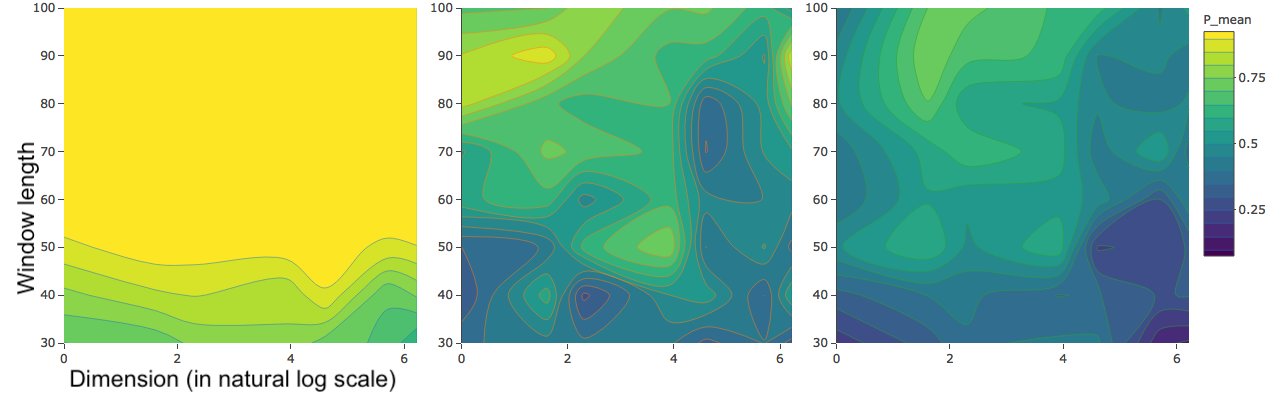}\setlength{\belowcaptionskip}{-5pt}
  \caption{Offline detection power P\_mean for mean change of $\Delta = 1/ \sqrt[3]{d}$, $d$ is the dimension of the data. Comparison between methods sGSR$_{CG}$, sGSR$_{MST}$, and IBGEC with respect to dimension and window length, with significance less than 5$\%$.}
 \label{fig:Contour_Static}
\end{figure}
\begin{figure}[t]
 \hspace {6pt}   sGSR$_{CG}$ \hspace{80pt}       sGSR$_{MST}$  \hspace{60pt}   IBGEC
  \centering 
   \includegraphics[height=.15\textheight]{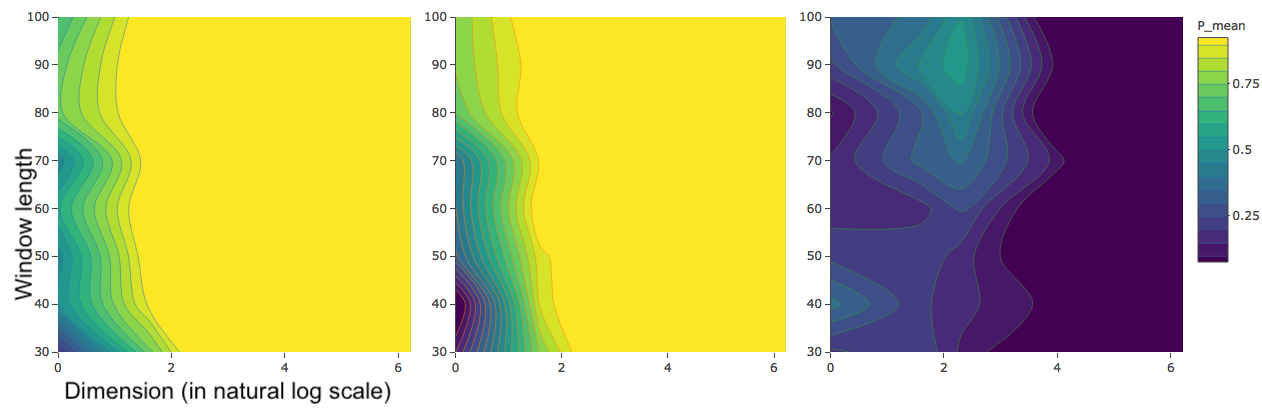} \setlength{\belowcaptionskip}{-5pt}
  \caption{Offline detection power P\_mean for variance change of $\Sigma = 2 I_d$. Comparison between methods sGSR$_{CG}$, sGSR$_{MST}$, and IBGEC with respect to dimension and window length, with significance less than 5$\%$.}
 \label{fig:Contour_StaticV}
\end{figure}
\begin{figure}[b]
(a) \hspace {3pt}  oGSR$_{CG}$ \hspace{30pt}   oGSR$_{MST}$ \hspace{50pt}  (b) \hspace {6pt}  oGSR$_{CG}$ \hspace{30pt}   oGSR$_{MST}$  
  \centering
  \includegraphics[height=.15\textheight]{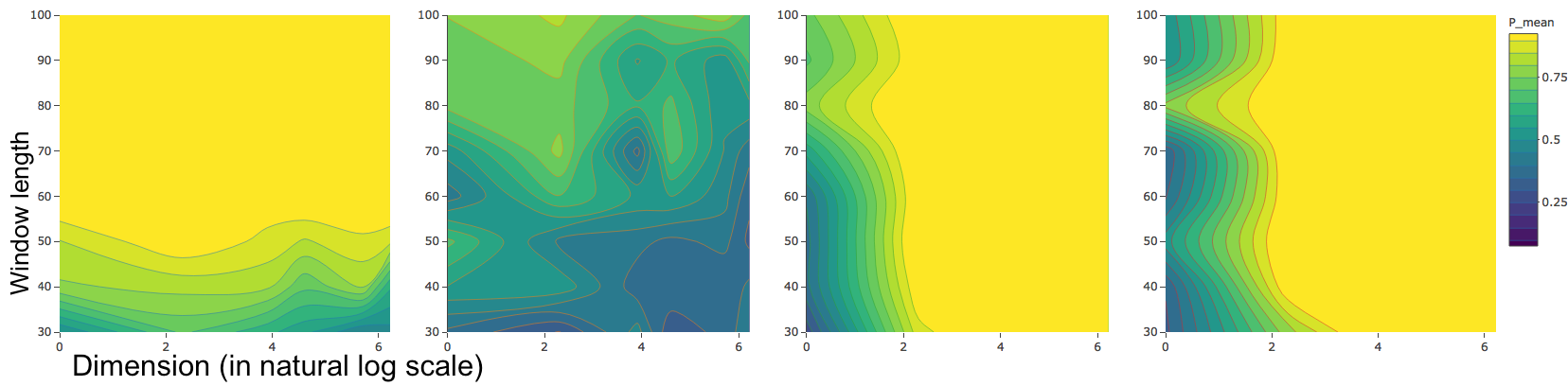}\setlength{\belowcaptionskip}{-5pt}
  \caption{Online detection power  P\_mean: (a) mean change of $\Delta = 1/ \sqrt[3]{d}$, $d$ is the dimension of the data and (b) variance change of $\Sigma = 2 I_d$. Comparison between oGSR$_{CG}$, oGSR$_{MST}$ with respect to dimension and window length,  with significance less than 10$\%$.}
 \label{fig:Contour_Online}
\end{figure}

In this numerical analysis, we generate 100 samples, in each sample it contain either with or without change-point in the middle of sample, as shown in (\ref{eq:simu_offline}). Note that for this power determination simulation, each sample is of the same length of the scanning window $n$, so in this setting, the change-point occurs at location $n/2$ of the data, if there is any. The purpose of this simulation is to compare the power of our proposed method with that of recent graph based change-point detection algorithm mentioned in Section \ref{subsec:Related}, which is an offline graph based method using in-between-group edge counting (IBGEC) algorithm. We compare the accuracy results from proposed staticGSR algorithm using graph choices of CG and MST, denote as sGSR$_{CG}$ and sGSR$_{MST}$, respectively, to result from in-between-group edge counting algorithm (IBGEC) by Chen and Zheng et al. We depict the detection power with respect to dimensionality of data $d$, and size of window $n$. Figure \ref{fig:Contour_Static} is the detection result of a mean change $\Delta=1/ \sqrt[3]{d}$. The detection power are higher for sGSR$_{CG}$ compared to IBGEC method, across all dimension and window length. In Figure \ref{fig:Contour_StaticV}, a significant improvement in detecting variance change, in particular, with small window length under high-dimensional scenarios. This make our proposed algorithm more ideal for further online detection, where a timely detection of change-point is important.

\subsection{Detection power for the online graph change point detection}
\begin{table}[t] 
\centering{
  \caption{Online detection power for mean change of $\Delta = 1/ \sqrt[3]{d}$, $d$ is the dimension of the data. Detect with multiple scanning window length of $\{40, 70,100\}$}\label{tb:OnlineM}
\begin{tabular}{ccrrrrrr}
 \toprule 
  &d           & 1    & 10   & 50   & 100  & 300  & 500  \\
 \toprule 
oGSR$_{CG}$ &P\_mean  & 0.97 & 0.98 & 0.98 & 0.97 & 0.99 & 0.98 \\
& FPR    & 0.08 & 0.09 & 0.09 & 0.10 & 0.05 & 0.05 \\
\cmidrule(r){1-8}  
 oGSR$_{MST}$& P\_mean & 0.60 & 0.64 & 0.71 & 0.53 & 0.52 & 0.66 \\
& FPR    & 0.10 & 0.11 & 0.13 & 0.14 & 0.06 & 0.16 \\
\bottomrule
\end{tabular}
}
\end{table}

\begin{table}[t]
\centering{
  \caption{ Online detection power for variance change of $\Sigma = 2 I_d$. Detect with multiple scanning window length of $\{40, 70,100\}$}\label{tb:OnlineV}
\begin{tabular}{ccrrrrrr}
 \toprule 
  &d           & 1    & 10   & 50   & 100  & 300  & 500  \\
 \toprule 
oGSR$_{CG}$ &P\_mean  & 0.50 & 0.98 & 0.98 & 0.97 & 0.98 & 0.98 \\
& FPR    & 0.09 & 0.07 & 0.09 & 0.12 & 0.08 & 0.07 \\
\cmidrule(r){1-8}  
 oGSR$_{MST}$& P\_mean& 0.46 & 0.97 & 0.97 & 0.96 & 0.98 & 0.96 \\
& FPR    & 0.11 & 0.11 & 0.12 & 0.16 & 0.07 & 0.13 \\
\bottomrule
\end{tabular}
}
\end{table}
In this section, we examine the online detection power of the proposed onlineGSR algorithm with CM and MST graphs (oGSR$_{CG}$,oGSR$_{MST}$). We generate 1000 samples for testing, each sample is consist of 100 observations, with first-half of observations follow $d$ dimensional standard normal distribution, and second-half of observations, with equal probability, either remains the same distribution or has a change in its distribution. For each dimension and each window length, we examine the power within the sample. In general, as seen in Figure \ref{fig:Contour_Online}, OnlineGSR algorithm with complete graph generally demonstrates better testing power under the same significance level. Furthermore, we study the online detecting power and false alarm rate, i.e. FPR, with multiple scanning window length. Scanning window from short- to long- length $\{40, 70, 100\}$ are applied into the online detection algorithm. In Table \ref{tb:OnlineM} and \ref{tb:OnlineV}, OnlineGSR algorithm with MST graph demonstrates high detection power and low false alarm rate for high-dimensional data.

\section*{Conclusion}
We proposed a novel graph-spanning ratio algorithm for change-point problem for data from low-to high- dimension. Comparing to a recent offline literature, our approach is sensitive to both mean and variance change. Among graph selections, our algorithm with complete graph delivers a high detecting power while maintain nominal false alarm rate. An important observation is that our graph-spanning approach has desirable power with small and multiple scanning windows, which enables online timely detection to change-point. Our approach is completely data-driven and the implementation can be extended to real-time data, such as finance stocks, human microbiome \cite{kuleshov2016synthetic}, and smart home data, where the problem of change detection is high-dimensional oriented.


\bibliographystyle{alpha}
\bibliography{References.bib}

\newpage
\begin{appendices}
\section{Nomenclature}
$y_i: i=1, 2, \dots. :$ sequential observations \\
$d:$ dimension of data \\
$n:$ size of scanning window\\
$N:$ tranning sample size for online GBCPD (Graph Based Change-Point Detection) \\
$\tau:$ location of change-point  \\
$F_0:$ distribution prior the change-point, i.e. distribution under the null hypothesis \\
$F_1:$ distribution after the change-point, i.e. distribution under the alternative hypothesis  \\
$G = (V,E):$ undirected graph consists vertices $V$ and edges $E$  \\
$W_{ij}:$ graph edge weight  \\
$G :$ graph constructed using all data points from the scanning window  \\
$G_1:$ graph constructed using data points from the first half of the scanning window, i.e. before change-point candidate  \\
$G_2: $graph constructed using data points from the second half of the scanning window, i.e. after change-point candidate   \\
$dG:$ sum of squared Euclidean distance of $G$ with window size of $n$    \\
$dG_1:$ sum of squared Euclidean distance of $G_1$ with window size of $n$  \\
$dG_2:$ sum of squared Euclidean distance of $G_2$ with window size of $n$  \\
$T_{\mu,n}:$ test statistic for graphical mean change with window size of $n$  \\
$T_{\sigma,n}:$ test statistic for graphical variance change with window size of $n$  \\
$\alpha:$ significant level of hypothesis testing  \\
$\rho_{\alpha,\mu,n}:$ critical value for test statistic $T_{\mu,n}$  \\
$\rho_{\alpha,\sigma,n}:$ critical value for test statistic $T_{\sigma,n}$ \\ 
$T_{\mu,n}^b:$ test statistic for graphical mean change with window size of n generated from permutation distribution \\
$T_{\sigma,n}^b:$ test statistic for graphical variance change with window size of n generated from permutation distribution  \\
$Z_{\mu,n}^b(z) :$ quantile function of $T_{\mu,n}^b$  \\
$Z_{\sigma,n}^b(z) :$ quantile function of $T_{\sigma,n}^b$  \\ 
$\rho^b_{\alpha,\mu,n}:$  critical value for test statistic $T_{\mu,n}$ from permutation distribution \\ 
$\rho^b_{\alpha,\sigma,n}:$  critical value for test statistic $T_{\sigma,n}$ from permutation distribution \\ 

\newpage
\section{Theoretical property and proof}
	We assume $ N$ is the number of observations and $ n $ is the scanning window size, where $ n < N $ and $ n $ is an even number.
	\begin{assumption}\label{A_gr}
		$ G$ is a complete graph (clique) with $ n $ nodes. By this we mean that the graph $ G $ is undirected and every pair of distinct vertices is connected by a unique edge.
	\end{assumption}
	
	\begin{assumption}\label{A_nodes}
		The nodes of  graph  $ G $, namely $ \by_i$, are i.i.d. random variables normally distributed, $ \by_i \sim\mathcal{N}\left(\mu, \sigma^{2} I \right)$ and $ \by_i \in \R$.
	\end{assumption}

\begin{assumption}\label{A_dist}
	The nodes are paired with the euclidean distance in $ \R $. 
\end{assumption}
	Next, we show some theoretical properties concerning the distribution of the GSR test statistics and the spanning distance of the graph.
	\begin{lemma} \label{A_spdidi}
		Let Assumptions \ref{A_gr}-\ref{A_dist} hold. Then, the distance spanned by $ G $ 
		\[ 
		dG = \sum_{i<j} \|\by_i, -\by_j\|^{2} = \sum_{i= 1}^{m} \lambda_{i} U_{i}^{2}
		 \]
		is  a linear combination of independent chi-square distributions, where $\|\cdot\|$ is the euclidean distance in $ \R $,$ U $ follows a multivariate standard normal distribution. 
	\end{lemma}
	\begin{proof}
		We define  by $ \w_{i,j} = ( \by_i - \by_j )$ to be a vector in $ \R $.  Then by Assumption \ref{A_nodes}, $ \w_{i,j}\sim \mathcal{N}(0, 2\sigma^{2}I) $ for $ i<j $ and $ i, j = 1,2, \dots, n $. Using the fact that the graph is fully connected , we construct $W = (\w_{i,j})$ to be a vector $ m\times 1 $ where $ m=$ $ {n}\choose {2}$$\cdot d $ and ${n}\choose {2}  $ is the total number of edges. Let $ \Sigma $ be the covariance matrix of $ W $. We define by  $ Z = \Sigma ^{-1/2}W $ to have zero mean and identity as covariance matrix. We are interested in the quadratic form
	\begin{equation}\label{A_spdec}
	dG = \sum_{i<j}  \big(\by_i - \by_j\big)^{\top}\big( \by_{i}-\by_{j}\big)= W^{\top}W=Z^{\top}\Sigma^{1/2} \Sigma^{1/2} Z.
	\end{equation}
	Henceforth, applying spectral decomposition theorem we obtain that $ \Sigma^{1/2}\Sigma^{1/2}= V^{\top}DV $, where $ D $ is the eigenvalue diagonal matrix of $ \Sigma $ with $ \lambda_{i} $ the eigenvalues of $\Sigma $ for $ i = 1, \cdots, m $ and $ V $ is the orthogonal matrix of the eigenvectors. Further,
	\begin{equation}\label{A_sp2}
	dG = Z^{\top}V^{\top}DV Z= U^{\top}DU
	\end{equation}
	where $ U= VZ $ and $ U $ is multivariate normal with mean zero and identity covariance matrix. The interested reader should refer to Chapter 4 of \cite{provost1992quadratic} for quadratic representations of multivariate normal distributions. To sum up, we have,
	\[
	dG = \sum_{i<j}W_{i,j}^{2} = W^{\top}W = U^{\top}DU= \sum_{i= 1}^{m}\lambda_{i}U_{i}^{2}.
	\]
	where $ U $ follows a multivariate standard normal distribution. As a result $dG$ is a linear combination of independent chi-squared random variable and the proof now is complete.
\end{proof}

\begin{property}
The distance spanned by $ G $ follows chi-square distribution with degrees of freedom $ (m-1)d $.
\end{property}
Further, we give the distribution ot the GSR mean test statistic $T_{\mu,n}$.
\begin{lemma}\label{A_testmu}
Let Assumptions \ref{A_gr} - \ref{A_dist} hold. Then, the GSR test statistics for  the graphical mean follows Fisher distribution. Precisely, 
\[
T_{\mu, n} = \frac{dG}{dG_{1} + dG_{2}}\sim F(d_{1}, d_{2}),
\]
where $ dG, dG_{1}, dG_{2}$ are the distances spanned by a fully-connected graph without change-point, before and after a potential change-point respectively. $ d_{1}, d_{2} $ are the degrees of freedom.
	\end{lemma}

	\begin{proof}
	Recall that we perform the GSR test statistics at sequential scanning windows of size $ n $. Following the notation from Lemma \ref{A_spdidi}, we define by $ dG = \frac{1}{2}\sum_{i, j} \|\by_{i} -\by_{j}\|^{2} $.
	After some trivial calculations we have that 
	\begin{equation}
	\begin{aligned}
	dG = & \frac{1}{2}\sum_{i, j} \|\by_{i} -\by_{j}\|^{2} = \frac{1}{2}\bigg(\sum_{i, j}\|\by_{i}\|^{2}+\|\by_{j}\|^{2}- 2 \by_{i}^{\top}\by_{j} \bigg)\\
	=& \frac{1}{2}\bigg( n \sum_{i}\|\by_{i}\|^{2} + n\sum_{j}\|\by_{j}\|^{2}-2 \sum_{i,j}\by_{i}^{\top}\by_{j}\bigg)\\
	=& \frac{1}{2}\bigg( n\sum_{i}\|\by_{i}\|^{2}-\left\|\sum_{i}\by_{i}\right\|^{2}\bigg).
	\end{aligned}	
		\end{equation}
Observe that $n\sum_{i}\|\by_{i}\|^{2}$ follows chi-square distribution with $ d(n-1) $ degrees of freedom \cite{spokoiny2015basics}. In light of the above calculations, the fact that we perform the GSR test statistics  in scanning window of size $ n$ and under the null hypothesis, we get that
\begin{equation}\label{A_ratio}
\frac{dG}{dG_{1}+dG_{2}}= \frac{\sum\limits_{i =\tau-n}^ {\tau +n}\|\by_{i}\|^{2}-\frac{1}{n}\left\|\sum\limits_{i= \tau-n}^{\tau +n}\by_{i}\right\|^{2}}{\sum\limits_{i =\tau-n}^{\tau+n}\|\by_{i}\|^{2}-\frac{1}{n}\left\|\sum\limits_{i= \tau-n}^{\tau}\by_{i}\right\|^{2}-\frac{1}{n}\left\|\sum\limits_{i= \tau}^{\tau +n}\by_{i}\right\|^{2}}.
\end{equation}
We set $ \xi = \frac{1}{\sqrt{n}}\sum\limits_{\tau-n}^{\tau+n}\by_{i} $, $ \xi_{-} = \frac{1}{\sqrt{n}}\sum\limits_{\tau-n}^{\tau}\by_{i} $,  $ \xi_{+} = \frac{1}{\sqrt{n}}\sum\limits_{\tau}^{\tau+n}\by_{i} $, $ S = \sum\limits_{\tau-n}^{\tau+n}\|\by_{i}\|^{2} $. Hence, (\ref{A_ratio}) takes the form 
\begin{equation}\label{A_ratio2}
\frac{dG}{dG_{1} +dG_{2}} = \frac{S-\|\xi_{-}+\xi_{+}\|^{2}}{S- \|\xi_{-}\|^{2}-\|\xi_{+}\|^{2}}.
\end{equation}
We notice that $ S = \|\xi_{-}\|^{2} +\|\xi_{+}\|^{2} +S_{rem} $, where $ S_{rem} $ are some remainder terms of $ S $. Thus, the GSR mean test statistic will follow Fisher distribution. Indeed, 
\begin{equation}
\frac{dG}{dG_{1} +dG_{2}} = \frac{\|\xi_{-}\|^{2} +\|\xi_{+}\|^{2} +S_{rem}-\|\xi_{-}+\xi_{+}\|^{2} }{S_{rem}}.
\end{equation}

\end{proof}

\begin{lemma}
	Let Assumptions \ref{A_gr} - \ref{A_dist} hold. Then, the GSR test statistic for the graphical variance follows the the distribution, 
		\[
	T_{\sigma, n} = \frac{dG_1}{dG_{2}} + \frac{dG_2}{dG_1} = Z +\frac{1}{Z} 
	\]
	where $dG_{1}, dG_{2}$ are the distances spanned by a fully-connected graph before and after a potential change point respectively. Thus, $ Z \sim F(d_{3}, d_{4})$ and $ \frac{1}{Z}\sim F(d_{4}, d_{3})$. $d_{3}, d_{4}$ are the degrees of freedom. 
\end{lemma}	
	\begin{proof}
		We denote by $  Z = \frac{dG_1}{dG_{2}} $ and $ \frac{1}{Z}=\frac{dG_2}{dG_{1}}   $. Also, following the notation of Lemma \ref{A_testmu} we have that
		\begin{equation}
		T_{\sigma, n} = Z + \frac{1}{Z} = \frac{\|\xi_{-}\|^{2}}{\|\xi_{+}\|^{2}} + \frac{\|\xi_{+}\|^{2}}{\|\xi_{-}\|^{2}}.
		\end{equation}
		By Lemma \ref{A_testmu}, $ \|\xi_{-}^{2}\|, \|\xi_{+} \|^{2}$ follows chi-squared distribution. As a result, $Z $ follows Fisher distribution with $(d_3, d_4)$ degrees and $ \frac{1}{Z} $ follows Fisher distribution with $(d_{4}, d_{3}) $ degrees of freedom. The proof now is complete.
		\end{proof}

\end{appendices}

\end{document}